%% file: main.tex
\documentclass{dukecei}

\usepackage{microtype}
\usepackage{graphicx}
\usepackage{subcaption}
\usepackage{booktabs} % for professional tables

\usepackage{hyperref}

\usepackage{amsmath}
\usepackage{amssymb}
\usepackage{mathtools}
\usepackage{amsthm}
\usepackage{booktabs}

\title{Swimba: Switch Mamba Model Scales State Space Models}

\author{
\centering
    Zhixu Du$^{1}$\footnote{Correspondence E-mail: zhixu.du@duke.edu},
    Krishna Teja Chitty-Venkata$^{2}$,
    Murali Emani$^{3}$,
    Venkatram Vishwanath$^{3}$\\
    Hai Helen Li$^{1}$,
    Yiran Chen$^{1}$\\[1em]
    \normalsize $^{1}$Duke University \\
    \normalsize $^{2}$Red Hat, Inc. \\
    \normalsize $^{3}$Argonne National Laboratory\\
}

\begin{document}

\maketitle
\thispagestyle{firstpagestyle} % Draws the header on the first page

\begin{abstract}
Mixture-of-experts (MoE) is a common approach for increasing parameter capacity, but applying MoE to state space model (SSM) token mixers can multiply the cost of the recurrent state update. We study how to introduce expert specialization into selective SSMs while preserving computational efficiency. We show that MoE--SSM can refer to two designs: (1) MoE over separated SSMs, which maintains multiple state trajectories and thus scales compute with the number of experts; and (2) MoE-parameterized SSM, which mixes experts in parameter space, maintains a single state trajectory, and evaluates the recurrence once. Our method, Switch Mamba (Swimba), follows the second design by routing over expert-produced SSM streams. Theoretically, we establish well-definedness and stability for MoE-parameterized SSMs and characterize the relationship between the two designs. Empirically, we evaluate Swimba on standard benchmark tasks and measure real-time throughput and latency. Under matched FLOPs, Swimba achieves slightly better average performance than the baseline, with a small slowdown in real-time latency and throughput. Overall, these results suggest that parameter-space MoE can increase SSM capacity while keeping the dominant recurrence cost fixed.
\end{abstract}

\input{src/introduction}

\input{src/background}
\input{src/swimba}
\input{src/experiments}

\input{src/related_work}
\input{src/conclusion}

\section*{Acknowledgments} 
This research is supported by NSF-2112562 and ARO W911NF-23-2-0224. This research used resources of the Argonne Leadership Computing Facility, which is a U.S. Department of Energy Office of Science User Facility operated under contract DE-AC02-06CH11357.

\bibliographystyle{abbrvnat}
\bibliography{reference}

\newpage
\appendix
\onecolumn
\input{src/appendix}

\end{document}

%% file: src/introduction.tex
\section{Introduction}
\label{sec:introduction}
State space models (SSMs) have become a practical alternative to attention for long-sequence modeling, combining recurrence with modern accelerator-friendly implementations. Recent selective SSM architectures such as Mamba \cite{gu2023mamba} and its SSD reformulation in Mamba-2 \cite{dao2024transformers} show that linear-time token mixing can be competitive with full attention. At the same time, scaling language models often relies on mixture-of-experts (MoE) to increase parameter count without proportionally increasing computational overhead \cite{shazeer2017outrageously,fedus2022switch}. Most MoE work targets feed-forward blocks, where expert activation does not replicate a costly recurrent state update, with only a few works applying MoE to attention \cite{zhang2022mixture, csordas2024switchhead}.
Extending MoE to SSM token mixers can retain the $O(L)$ computational complexity while boosting layer capacity through MoE scaling. However, it introduces a different challenge. The core recurrence is the dominant cost in SSMs, so naive implementation can scale compute with the number of experts.

\begin{figure}[t]
    \centering
    \includegraphics[width=0.4\textwidth]{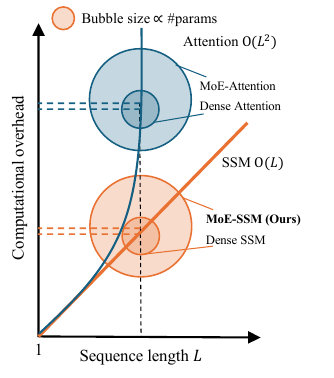}
    \caption{Compute scaling view for token mixing. The x-axis is sequence length $L$ and the y-axis is per-token computational overhead. Attention scales as $O(L^2)$ while SSMs scale as $O(L)$. Bubble area indicates parameter count, illustrating that MoE can increase parameters with only a small increase in per-token compute, motivating Swimba (MoE-SSM) designs.}
    \label{fig:placeholder}
\end{figure}

\textit{MoE-SSM} can refer to two different designs, and the distinction matters for both efficiency and modeling behavior. \textbf{MoE of separated SSMs} maintains an expert-specific state trajectory and advances multiple recurrences in parallel. This resembles switching dynamics in state-space models \cite{ghahramani2000variational}: each expert can store its own memory, but compute and state storage grow with the number of experts. In contrast, \textbf{MoE-parameterized SSM} keeps a single state trajectory and mixes experts in parameter space. Instead of running multiple recurrences and combining outputs, it forms one effective selective SSM by aggregating expert-dependent injection and readout streams, then evaluates the recurrence once with an SSD-style computation \cite{dao2024transformers}.

Much of the prior MoE--SSM hybrid literature increases capacity by interleaving MoE MLP blocks with dense SSM mixers \cite{lieber2024jamba,pioro2024moe,anthony2024blackmamba}. For approaches that apply MoE within the SSM block, existing work is largely engineering-driven and often does not distinguish between the two MoE--SSM designs \cite{zhan2025routing}. Without a clear taxonomy, it is easy to conflate these designs and their scaling properties.

This paper focuses on MoE-parameterized SSMs. We present \textbf{Switch Mamba (Swimba)}, an MoE-parameterized SSM layer that preserves a single-pass recurrence. Swimba builds on Mamba-2~\cite{dao2024transformers}: each expert produces candidate selective SSM streams, and a token-level router computes mixture weights. We then mix the expert-dependent SSM streams in parameter space and run a single SSM evaluation to update the recurrence. In our implementation, specialization enters through routing-weighted SSM streams, while state evolution remains a single trajectory. This design targets the central efficiency goal of SSMs: it avoids replicating the expensive recurrence across experts, while scaling the number of model parameters.

We also provide formal theoretical statements that motivate the MoE-parameterized SSM design choice. Specifically, Theorem~\ref{thm:moe_structure} shows that mixing in parameter space preserves the single-SSM structure required by Mamba-2 for efficiency. Theorem~\ref{thm:moe_complexity} shows that the recurrence cost does not scale with the number of experts. Under a contractive transition, Theorem~\ref{thm:moe_stability} shows that stability can be controlled through bounds on the mixed streams. Finally, we relate MoE-parameterized SSMs to the MoE of separated SSMs baseline: Theorem~\ref{thm:moe_approx} identifies regimes where the two agree, and Theorem~\ref{thm:moe_expressive} characterizes the mismatch introduced when routing varies over time, while still allowing increased expressivity from input-dependent mixtures.

Empirically, we build Swimba on the Nemotron-H-8B \citep{blakeman2025nemotron} hybrid backbone, replacing each Mamba-2 token-mixing layer with our MoE-parameterized SSM layer while keeping the rest of the architecture unchanged. We evaluate against the Nemotron-H-8B baseline on standard benchmark tasks. We compare model FLOPs to the baseline and benchmark end-to-end decoding with \texttt{vLLM} \citep{kwon2023efficient} under matched settings. Across these evaluations, Swimba achieves slightly better average performance than the baseline at comparable total FLOPs, while showing a small slowdown in real-time latency and throughput. Prior results~\citep{chitty2025moe} indicate that increasing the number of experts while keeping the number of active experts fixed has limited impact on latency and throughput, suggesting that Swimba retains favorable scaling behavior. 

Our contributions can be summarized as follows:
\begin{itemize}
\item We distinguish two MoE--SSM designs, \emph{MoE of separated SSMs} and \emph{MoE-parameterized SSMs}, and explain how this distinction determines scaling in compute and memory. We provide theoretical results on well-definedness, stability, and the relationship between the two designs.
\item We introduce Swimba, an MoE-parameterized SSM layer that mixes expert-dependent streams in parameter space and preserves SSM computation as a single recurrence evaluation.
\item We evaluate Swimba on standard benchmarks and report compute- and decoding-oriented measurements (FLOPs analysis and \texttt{vLLM} throughput/latency) under matched settings. Swimba achieves slightly better average performance than the baseline at comparable runtime cost, suggesting strong potential for scaling.
\end{itemize}

%% file: src/background.tex
\section{Background}
\textbf{Notation.} The input sequence has length $T$, with $X\in\mathbb{R}^{T\times P}$ and output $Y\in\mathbb{R}^{T\times P}$. We write $X_t\in\mathbb{R}^{P}$ and $Y_t\in\mathbb{R}^{P}$ for the $t$-th token. We use $h_t\in\mathbb{R}^{N\times P}$ for the latent state, which can be viewed as $P$ independent states $\{h_{t,p}\in\mathbb{R}^{N}\}_{p=1}^P$ broadcast across channels. Parameters $A$, $B$, and $C$ have shapes $\mathbb{R}^{N\times N}$, $\mathbb{R}^{N\times P}$, and $\mathbb{R}^{N\times P}$, respectively. Parameters $A_t$, $B_t$, and $C_t$ refers to the time-varying parameters at time $t$, sharing the same shape as $A$, $B$, and $C$. We use $E$ to denote the total number of experts in an MoE model and $E_e(\cdot)$ to denote the $e$-th expert function.

\subsection{State Space Models}
State Space Models (SSMs) provide a framework for modeling sequences by discretizing a continuous-time system. The system maps a 1-D input function $x(t) \in \mathbb{R} \to y(t) \in \mathbb{R}$ via a latent state $h(t) \in \mathbb{R}^N$ evolving according to the differential equation $h'(t) = {A}h(t) + {B}x(t)$. 
In deep learning, this system is discretized (typically via Zero-Order Hold) with a timescale parameter $\Delta$, resulting in the recurrence:
\begin{equation}
    h_t = \overline{{A}} h_{t-1} + \overline{{B}} x_t, \quad y_t = {C}^\top h_t,
\end{equation}
where the discrete parameters are $\overline{{A}} = \exp(\Delta {A})$ and $\overline{{B}} = (\Delta {A})^{-1}(\exp(\Delta {A}) - {I}) \cdot \Delta {B}$.
Structured SSMs like S4~\cite{gu2021efficiently} constrain the transition matrix ${A}$ (e.g., to be diagonal plus low-rank) to enable efficient training via global convolutions.

\subsection{Mamba and State Space Duality (SSD)}
\textbf{Selective SSMs.} Mamba~\cite{gu2023mamba} introduced the \textit{selection mechanism}, making the parameters ${B}, {C},$ and $\Delta$ functions of the input $x_t$. This time-variance allows the model to filter information based on content but precludes the use of efficient convolutions, necessitating a hardware-aware parallel scan (prefix sum).

\textbf{Mamba-2 and SSD.} Dao and Gu~\cite{dao2024transformers} reformulated selective SSMs under the framework of \textit{State Space Duality} (SSD). They prove that if the transition matrix ${A}$ is structured as a scalar-identity matrix (i.e., ${A} = a{I}$), the recurrent SSM is mathematically dual to a structured attention mechanism. Specifically, the sequence transformation can be viewed as a matrix multiplication $Y = {M}X$, where ${M}$ is a semiseparable matrix masked by the decay $a$.
This duality allows the computation to be decomposed into chunks, utilizing efficient matrix multiplication (Tensor Cores) within chunks and linear recurrence between chunks. We denote this core operator as:
\begin{equation}
    Y = \mathrm{SSD}({A}, {B}, {C}, {X}).
\end{equation}
Crucially, this formulation maintains $\mathcal{O}(T)$ inference complexity while allowing for parallel training.

\subsection{Mixture of Experts and Switching Dynamics}
Mixture of Experts (MoE) scales model capacity without increasing inference cost by conditionally activating subsets of parameters \cite{shazeer2017outrageously}. A learnable router $R(x_t)$ computes a probability distribution $\pi_t$ over $E$ experts and selects a sparse active set of indices $\mathcal{K}t$ (Top-$k$). The output is a weighted sum of the active expert computations:
\begin{equation}
y_t = \sum_{e \in \mathcal{K}t} \pi_{t,e} E_e(x_t).
\end{equation}
We compute routing logits $g_t\in\mathbb{R}^E$ and routing probabilities $\pi_t$ via a softmax:
\begin{equation}
\pi_t=\mathrm{softmax}(g_t),\qquad \sum_{e=1}^E \pi_{t,e}=1.
\label{eq:router_softmax_method}
\end{equation}
We use $\pi_t$ directly as mixture weights. With top-$k$ sparsity, we keep the softmax values on the active set and set the rest to zero, without renormalization. In Transformers, MoE is typically applied to the feed-forward networks (FFN) \cite{fedus2022switch}.

% \textbf{Switching State-Space Models.} Applying MoE to dynamical systems relates to \textit{Switching SSMs}, where the system dynamics switch between different regimes over time~\cite{ghahramani2000variational}.
% Classically, if the transition matrix ${A}$ switches between experts (i.e., $h_t = {A}^{(e)} h_{t-1}$), the exact posterior distribution of the latent state becomes a Gaussian mixture with $E^T$ components, rendering inference computationally intractable.
% This motivates designs that mix parameters in a way that preserves a unimodal belief state, avoiding the exponential explosion associated with switching dynamics.

%% file: src/swimba.tex
% \section{Swimba}
% \subsection{Model Architecture}
% % model architecture 
% % theoretical foundation / support
% \subsection{Pre-Training}

\section{MoE-SSM}
\label{sec:method}

\subsection{Selective SSM and SSD}
\label{sec:ssm_ SSM}

A selective SSM defines the sequence transformation $X\mapsto Y$ by
\begin{equation}
h_t = A_t h_{t-1} + B_t X_t,\qquad
Y_t = C_t^\top h_t.
\label{eq:ssm_selective_method}
\end{equation}
Following Mamba-2's channelwise convention, \eqref{eq:ssm_selective_method} is understood per channel:
$h_{t,p}=A_t h_{t-1,p}+b_{t,p}x_{t,p}$ and $y_{t,p}=c_{t,p}^\top h_{t,p}$.
We denote by
\begin{equation}
Y=\mathrm{ SSM}(A,X,B,C)
\label{eq: SSM_operator_method}
\end{equation}
the structured state-space computation ( SSM) that evaluates \eqref{eq:ssm_selective_method} over $t=1,\dots,T$.
In this work, $\mathrm{ SSM}(\cdot)$ is unchanged.
Our method specifies how the streams $(A_t,B_t,C_t)$ and the effective injected input are produced from a mixture of experts, while the evaluation itself remains a single  SSM pass.

\subsection{MoE-parameterized  SSM}
\label{sec:moe_param_ SSM}

Each expert $e\in\{1,\dots,E\}$ provides candidate selective SSM streams
\[
\{A_t^{(e)},\,B_t^{(e)},\,C_t^{(e)},\,X_t^{(e)}\}_{t=1}^T.
\]
Rather than instantiating $E$ separate dynamical systems, we form a single selective SSM by mixing these expert-dependent quantities in parameter space using $\pi_t$.
As a result, the model maintains one state trajectory $\{h_t\}$ and executes $\mathrm{ SSM}$ once.

In our implementation, the expert-dependent streams $\{B_t^{(e)},C_t^{(e)},X_t^{(e)}\}$ are produced from the same token features $X_t$ using expert-specific linear projections, matching the Mamba-2 style of generating token-dependent parameters.
At the same time, we share the transition across experts and time:
\begin{equation}
A_t^{(e)}\equiv A\qquad \text{for all }e,t,
\label{eq:shared_A_method}
\end{equation}
so expert diversity enters through the injection and readout streams, not through separate transitions.

With routing probabilities $\pi_t$, we form one effective state update by mixing the injected input term and the readout:
\begin{equation}
h_t = A h_{t-1} +
\sum_{e\in\mathcal{K}_t}\pi_{t,e}\,B_t^{(e)}X_t^{(e)},
\label{eq:moe_state_method}
\end{equation}
\begin{equation}
Y_t =
\Big(\sum_{e\in\mathcal{K}_t}\pi_{t,e}C_t^{(e)}\Big)^\top h_t,
\label{eq:moe_out_method}
\end{equation}
where $\mathcal{K}_t$ is either all experts or the top-$k$ active set.
This is the critical point: we never create expert-specific states, so the recurrence is evaluated once, and the  SSM computation remains a single pass.

There are two natural ways to combine MoE with SSMs.
One option is to mix expert outputs after running multiple expert recurrences; the other is to mix expert-dependent parameters and keep a single recurrence.
We choose the second option because it preserves the main computational advantage of  SSM while still allowing token-dependent specialization through $(B_t^{(e)},C_t^{(e)},X_t^{(e)})$ and $\pi_t$.

\subsection{Alternative: MoE of separated SSMs}
\label{sec:moe_sep_ssm}

For completeness, we define MoE of separated SSMs as the design that maintains an independent state per expert:
\begin{equation}
h_t^{(e)} = A h_{t-1}^{(e)} + B_t^{(e)}X_t^{(e)},\qquad
Y_t = \sum_{e\in\mathcal{K}_t}\pi_{t,e}\,(C_t^{(e)})^\top h_t^{(e)}.
\label{eq:sep_ssm_method}
\end{equation}
This design can represent a mixture of distinct hidden trajectories, but it typically requires advancing multiple recurrences and storing expert-specific states, which directly increases memory and compute with the number of active experts.
Fig.~\ref{fig:MoE-SSMs} uses a causal figure to depict the difference between these two MoE--SSM designs.

\vspace{1em}
\begin{figure}[t]
  \centering

  \begin{subfigure}[b]{0.45\textwidth}
    \centering
    \includegraphics[width=0.80\textwidth]{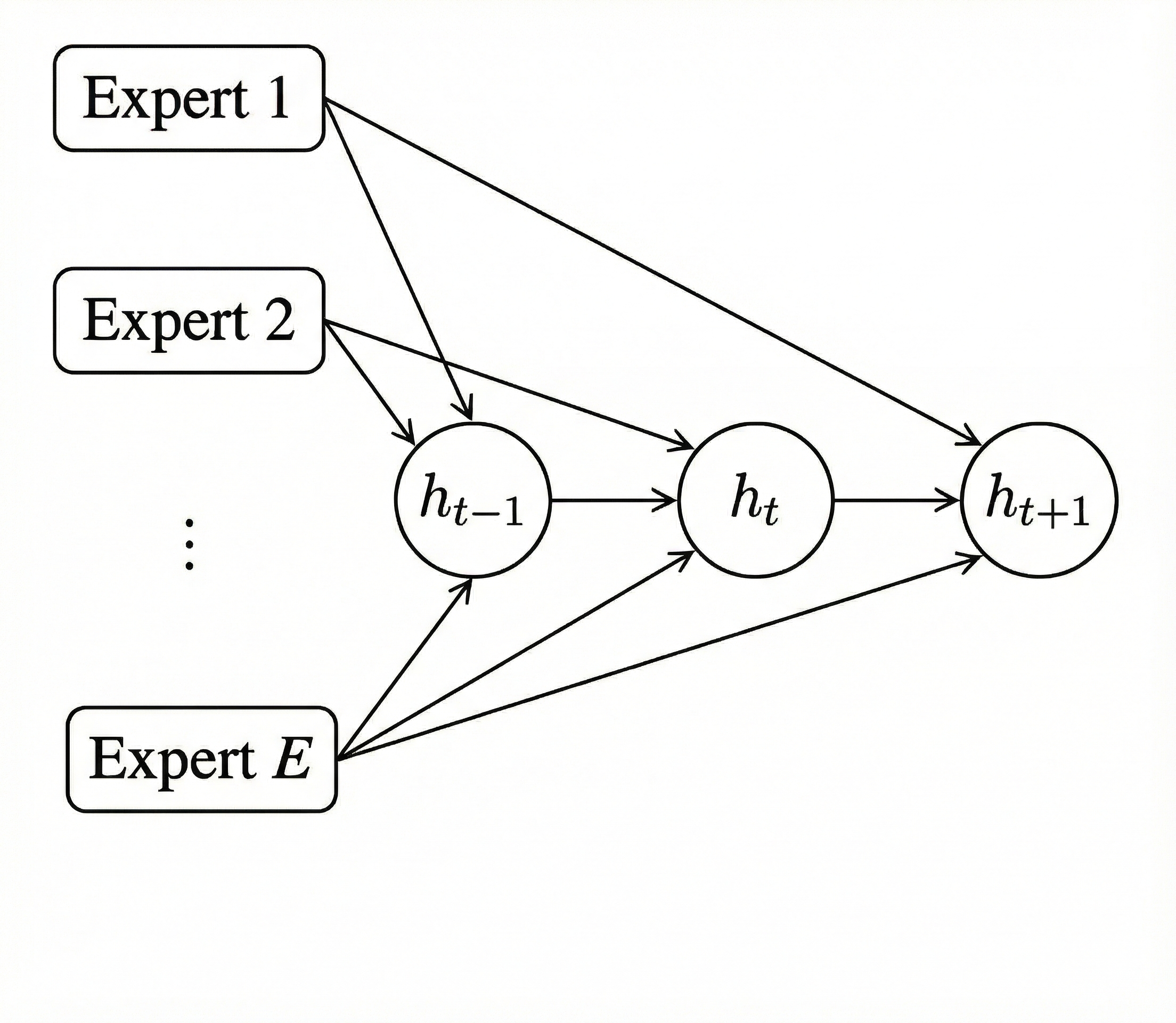}
    \vspace{-2em}
    \caption{}
    \label{fig:stack_a}
  \end{subfigure}
  \hfill
  \begin{subfigure}[b]{0.45\textwidth}
    \centering
    \includegraphics[width=0.75\textwidth]{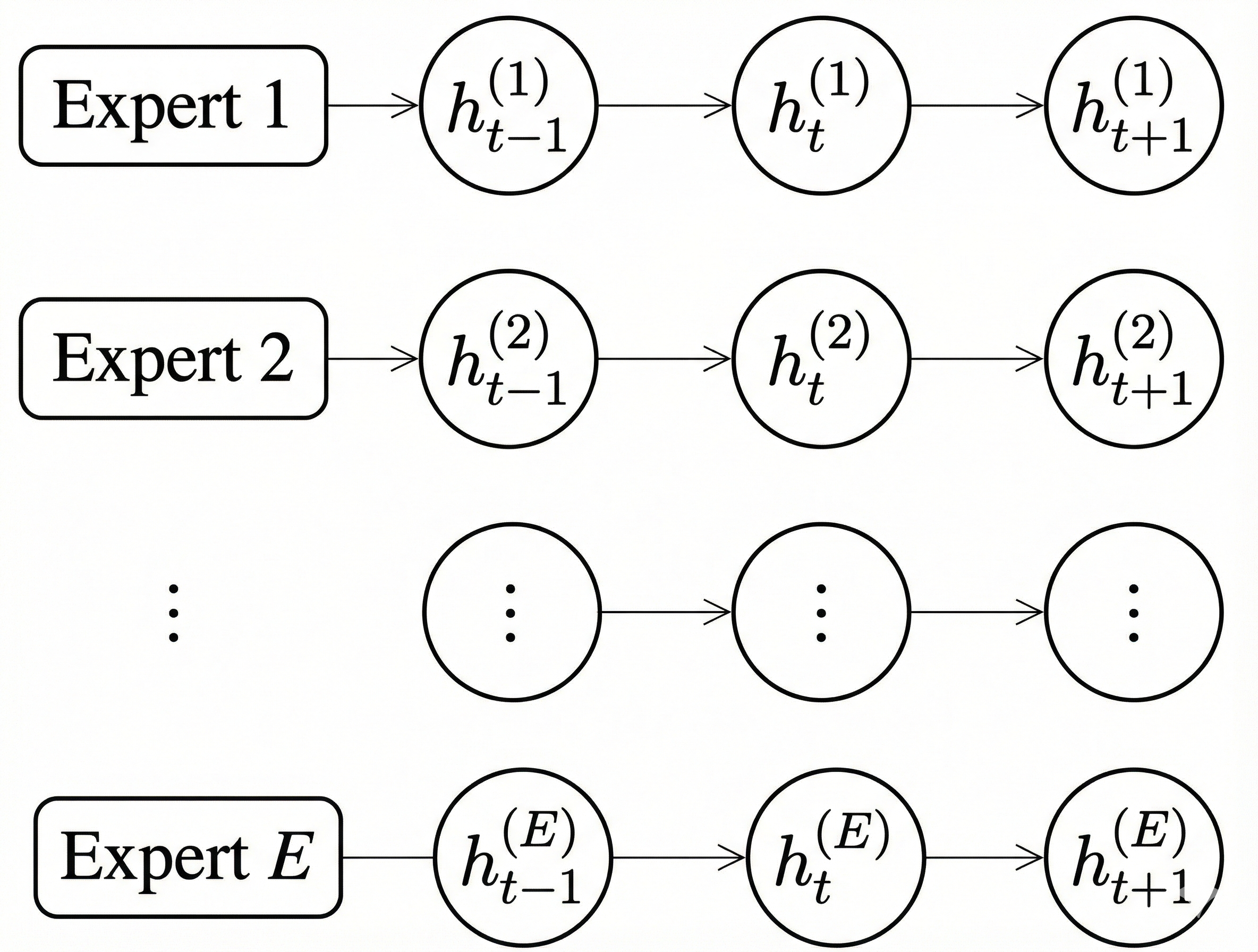}
    \vspace{0.5em}
    \caption{}
    \label{fig:stack_b}
  \end{subfigure}
  
  \caption{Two MoE--SSM designs with different state and compute scaling. (a) MoE-parameterized SSM (parameter-space mixing) keeps a single hidden-state trajectory and mixes expert-produced streams before a single recurrence evaluation. (b) MoE over separated SSMs maintains one state trajectory per expert and combines expert outputs, which requires advancing multiple recurrences resulting computation scales with number of experts.}
  \label{fig:MoE-SSMs}
\end{figure}

\subsection{Theoretical properties}
\label{sec:theory_in_method}

We now summarize the formal statements that motivate the design and defer full proofs to Appendix~\ref{sec:appendix_proofs}.
Throughout, routing changes the streams fed into  SSM, but it does not change the  SSM evaluation itself.

The first question is whether parameter-space mixing breaks the SSM structure required by  SSM.
It does not: after mixing, the layer is still a single selective SSM with state size $N$, which means we can reuse the same  SSM implementation and the same type of structural analysis as in Mamba-2.

\begin{theorem}[Single-SSM structure under MoE routing]
\label{thm:moe_structure}
Fix $T,N,P,E$.
Assume $A\in\mathbb{R}^{N\times N}$ and, for each $t$ and $e$, streams $B_t^{(e)},C_t^{(e)}\in\mathbb{R}^{N\times P}$ and $X_t^{(e)}\in\mathbb{R}^P$ are given.
Let $\pi_t\in\mathbb{R}^E$ satisfy $\pi_{t,e}\ge 0$ and let $\mathcal{K}_t\subseteq[E]$ denote the active set (dense or top-$k$).
Define mixed streams
\[
\widetilde U_t := \sum_{e\in\mathcal{K}_t}\pi_{t,e}\,B_t^{(e)}X_t^{(e)},\qquad
\widetilde C_t := \sum_{e\in\mathcal{K}_t}\pi_{t,e}\,C_t^{(e)}.
\]
Then \eqref{eq:moe_state_method}--\eqref{eq:moe_out_method} is exactly a single selective SSM
\[
h_t = A h_{t-1} + \widetilde U_t,\qquad
Y_t = \widetilde C_t^\top h_t,
\]
with state size $N$, independent of $E$ and independent of sparsity in $\mathcal{K}_t$.
\end{theorem}

This statement is immediate once the recurrence is written in terms of $\widetilde U_t$ and $\widetilde C_t$: MoE changes the effective injection and readout,  so the model stays in the same family as a standard selective SSM layer.

Next, we formalize the computational reason for mixing in parameter space.
The dominant cost in  SSM comes from evolving the state over $T$ steps; our design evolves that state once, whereas the separated alternative evolves it once per expert.

\begin{theorem}[Recurrence complexity does not scale with $E$]
\label{thm:moe_complexity}
Let $\mathcal{C}_{\mathrm{step}}(N,P)$ be the per-step cost of one  SSM/SSM state update with state size $N$ and $P$ channels.
Assume mixing the expert streams at time $t$ costs $\mathcal{C}_{\mathrm{mix}}(k,P)$ when at most $k$ experts are active.
Then the MoE-parameterized  SSM forward pass costs
\[
\mathcal{O}\!\left(T\,\mathcal{C}_{\mathrm{step}}(N,P)+T\,\mathcal{C}_{\mathrm{mix}}(k,P)\right).
\]
In contrast, a separated MoE that advances all $E$ expert states each step costs
\[
\mathcal{O}\!\left(T\,E\,\mathcal{C}_{\mathrm{step}}(N,P)\right).
\]
\end{theorem}

The comparison highlights the intended scaling: increasing the number of experts increases parameters, but it does not multiply the expensive part of the computation, because the recurrence is not replicated.
The additional work is confined to routing, expert projections, and the mixing operation.

A practical concern is whether rapid expert switching can destabilize the dynamics.
With a contractive transition matrix $A$, stability reduces to bounding the mixed injection and mixed readout streams.

\begin{theorem}[BIBO stability under a contractive transition]
\label{thm:moe_stability}
Assume $\|A\|\le\rho<1$ in an induced matrix norm.
Assume there exist constants $U,C>0$ such that for all $t$,
$\|\widetilde U_t\|\le U$ and $\|\widetilde C_t\|\le C$.
Then for any $h_0$ and all $t\in[T]$,
\[
\|h_t\|\le \rho^t\|h_0\| + \frac{1-\rho^t}{1-\rho}\,U,
\]
\[
\|Y_t\|\le C\left(\rho^t\|h_0\| + \frac{1-\rho^t}{1-\rho}\,U\right).
\]
\end{theorem}

The bound reflects a standard SSM behavior: a stable transition forgets the distant past at a geometric rate.
Routing can affect the state only through the injected term $\widetilde U_t$, so controlling that term is sufficient to prevent state blow-up.

Finally, we clarify how the parameter-mixed model relates to the separated MoE baseline.
The separated model can keep expert-specific memories through distinct trajectories $\{h_t^{(e)}\}$, whereas our model compresses the dynamics into one trajectory.
This compression is exact in a natural commutative regime and otherwise introduces a mismatch that can be related to per-step differences between each expert injection and their mixture.

\begin{theorem}[Separated vs parameter-mixed: equality regime and mismatch bound]
\label{thm:moe_approx}
Assume all experts share $A$ and initialize with $h_0^{(e)}=0$.
Let the separated model follow \eqref{eq:sep_ssm_method} and the mixed model follow \eqref{eq:moe_state_method}--\eqref{eq:moe_out_method}.
If $\pi_{t,e}\equiv \pi_e$ is time-invariant and $C_t^{(e)}\equiv C_t$ for all $e$, then $Y_t^{\mathrm{sep}}=Y_t^{\mathrm{mix}}$ for all $t$.
More generally, if $\|A\|\le\rho<1$ and $\|C_t^{(e)}\|\le C$, then
\[
\|Y_t^{\mathrm{sep}}-Y_t^{\mathrm{mix}}\|
\le
C\sum_{e=1}^E \pi_{t,e}\,\|h_t^{(e)}-h_t\|,
\]
and each difference $h_t^{(e)}-h_t$ satisfies a stable linear recursion driven by
$B_t^{(e)}X_t^{(e)}-\sum_j \pi_{t,j}B_t^{(j)}X_t^{(j)}$.
\end{theorem}

This theorem separates two effects.
When routing is fixed and readout is shared, mixing commutes with the linear recurrence and the two models agree exactly.
When routing varies, the separated model can preserve expert-specific information in its hidden states, while the mixed model aggregates those influences into one shared memory; the bound makes clear that the gap is controlled by how different each expert's injected input is from the mixture at each step.

The final question is whether parameter-space mixing increases modeling power or whether it is only an efficiency trick.
Even in a simplified setting where each expert stream is a linear function of $X_t$, the router introduces an input-dependent combination over multiple parameter maps.
This strictly enlarges the function class compared to a single-expert layer, while still requiring one recurrence.

\begin{theorem}[Strict expressivity gain with one recurrence]
\label{thm:moe_expressive}
Consider a one-step specialization ($T=1$) with a single state update and readout.
Assume each expert produces $(B_1^{(e)},C_1^{(e)},X_1^{(e)})$ as linear projections of $X_1$, and routing uses $\pi_1=\mathrm{softmax}(g_1)$ with $g_1$ a linear function of $X_1$.
Then the set of functions representable by the MoE-parameterized layer strictly contains the set representable by any single-expert linear-projection layer, while both execute one state update.
\end{theorem}

\section{Swimba (Switch Mamba) Architecture}
\label{sec:arch}
\begin{figure}[t]
    \centering
    \includegraphics[width=0.5\textwidth]{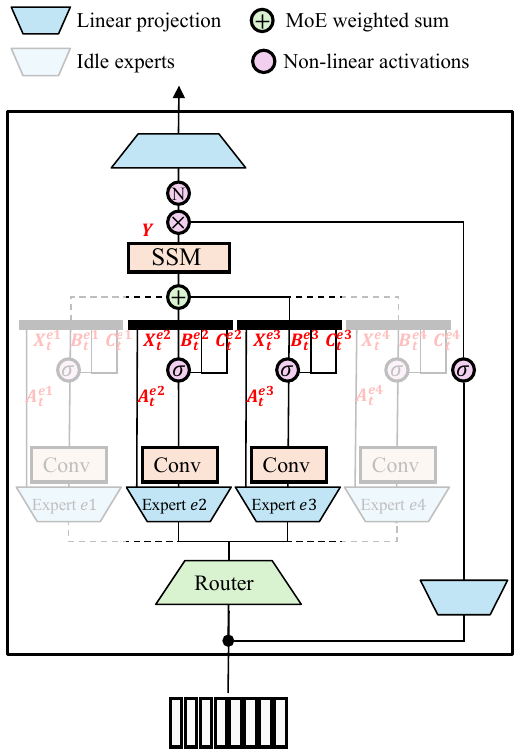}
    \caption{Swimba layer: an MoE-parameterized SSM token mixer built from a Mamba-2 block. Each expert applies its own input projections to produce candidate token-dependent streams (e.g., $B_t^{(e)}, C_t^{(e)}, X_t^{(e)}$). At each layer, only a subset of experts is selected per token. The selected experts are combined by a weighted sum to form a single stream for the SSM computation.}
    \label{fig:arch}
\end{figure}

\subsection{Swimba layer} 
We build the Swimba layer on top of Mamba-2~\citep{dao2024transformers}. Fig.~\ref{fig:arch} shows the Swimba layer architecture. In the original Mamba-2 layer, the in-projection maps each token's hidden state to a higher dimension and slices the result into the parameters $(A,X,B,C)$. To apply an MoE-parameterized SSM, we replace this in-projection with an MoE module: multiple in-projection linear layers serve as experts, and a router selects which expert to activate for each token. The activated experts produce expert-specific parameters, which are aggregated when more than one expert is activated. The aggregated parameters are then fed into the SSM module to perform the SSM computation. The rest of the Mamba-2 layer is kept unchanged.

For implementation, we preserve the Mamba-2 block wrapper and modify only how the  SSM inputs are produced. We follow Mamba-2's block design, which treats the inner mixer as $(A,X,B,C)\mapsto Y$ and produces $(A,X,B,C)$ in parallel at the beginning of the block.

\subsection{Backbone} 
We build a hybrid decoder-only language model using the Nemotron-H-8B~\citep{blakeman2025nemotron} backbone configuration and replacing every Mamba-2 layer with our MoE-parameterized SSM layer. The model has 52 layers and follows the same hybrid layout: four self-attention layers are dispersed through depth, and the remaining layers alternate between Swimba layers and FFN layers. The first layer is a sequence-mixing layer, the last layer is an FFN, and self-attention layers (when present) precede the FFNs they pair with. The only architectural change is the sequence mixer: each original Mamba-2 layer is replaced by a Swimba layer, while the rest of the backbone is kept unchanged.

\subsection{Pre-training}
\label{sec:pretraining}

We initialize from the Nemotron-H-8B checkpoint by copying all shared parameters and then introducing our MoE-parameterized SSM parameters (experts and routers) for each Swimba layer. During continued training, we update only the newly introduced expert and router parameters while keeping the copied backbone parameters fixed.
We perform supervised fine-tuning on the \texttt{Llama Nemotron Post-Training Dataset}~\citep{bercovich2025llama}, released as part of the Llama-Nemotron post-training resources. The dataset spans categories including math, code, science, instruction following, chat, and safety, and reports aggregate counts per category. We use only the SFT portion for teacher-forced training.
We follow the dataset’s standard schema: each example provides an input chat as a list of role-based messages and a target output string.

%% file: src/experiments.tex
\section{Experiments}
\label{sec:experiments}
\begin{table*}[t]
  \centering
  \small
  \setlength{\tabcolsep}{7pt}
  \renewcommand{\arraystretch}{1.15}
  \caption{R sesults on standard benchmarks. Subscripts report standard deviation. Nemotron-H-8B and Swimba-14B are compared at similar FLOPs per token. The last row reports the unweighted average across tasks (and across normalized accuracies where provided). Swimba-14B improves on the baseline on most tasks and on the average score.}
  \label{tab:main_results}
  \begin{tabular}{lcccc}
    \toprule
    & \multicolumn{2}{c}{Nemotron-H-8B} & \multicolumn{2}{c}{Swimba-14B} \\
    \cmidrule(lr){2-3}\cmidrule(lr){4-5}
    Task & Accuracy & Normalized accuracy & Accuracy & Normalized accuracy \\
    \midrule
    Arc-Challenge & $56.5_{\pm 1.4}$ & $60.3_{\pm 1.4}$ & $\mathbf{59.5}_{\pm 1.4}$ & $\mathbf{63.5}_{\pm 1.4}$ \\
    Arc-Easy      & $84.0_{\pm 0.8}$ & $83.7_{\pm 0.8}$ & $\mathbf{84.5}_{\pm 0.8}$ & $\mathbf{84.1}_{\pm 0.8}$ \\
    BoolQ         & $\mathbf{85.5}_{\pm 0.6}$ & \textemdash      & $84.9_{\pm 0.6}$ & \textemdash \\
    Hellaswag     & $61.9_{\pm 0.5}$ & $81.2_{\pm 0.4}$ & $\mathbf{64.2}_{\pm 0.5}$ & $\mathbf{84.9}_{\pm 0.4}$ \\
    MMLU          & $71.7_{\pm 0.4}$ & \textemdash      & $\mathbf{75.0}_{\pm 0.4}$ & \textemdash \\
    OpenbookQA    & $\mathbf{35.4}_{\pm 2.1}$ & $\mathbf{47.4}_{\pm 2.2}$ & $34.9_{\pm 2.1}$ & $46.7_{\pm 2.2}$ \\
    PIQA          & $81.7_{\pm 0.9}$ & $82.3_{\pm 0.9}$ & $\mathbf{82.0}_{\pm 0.9}$ & $\mathbf{82.5}_{\pm 0.9}$ \\
    RTE           & $71.8_{\pm 2.7}$ & \textemdash      & $\mathbf{73.2}_{\pm 2.7}$ & \textemdash \\
    WinoGrande    & $76.7_{\pm 1.2}$ & \textemdash      & $\mathbf{79.1}_{\pm 1.2}$ & \textemdash \\
    \midrule
    Average & $69.47$ & $70.98$ & $\mathbf{70.81}$ & $\mathbf{72.34}$\\
    \bottomrule
  \end{tabular}%
\end{table*}

\begin{table}[t]
  \centering
  \small
  \setlength{\tabcolsep}{10pt}
  \renewcommand{\arraystretch}{1.15}
  \caption{Per-token inference FLOPs. Swimba-14B has essentially the same FLOPs per token as Nemotron-H-8B, indicating that using a single selected expert per Swimba layer does not materially change the dominant inference compute.}
  \label{tab:per_token_flops}
  \begin{tabular}{lc}
    \toprule
    Model & FLOPs / token \\
    \midrule
    Nemotron-H-8B & $1.51263\times 10^{10}$ \\
    Swimba-14B    & $1.51282\times 10^{10}$ \\
    \bottomrule
  \end{tabular}
\end{table}

We empirically evaluate the effectiveness and efficiency of Swimba and report the results in this section.

\textbf{Setup.} We evaluate Swimba-14B against the Nemotron-H-8B baseline. Swimba-14B is configured with 4 experts per Swimba layer, and each token activates a single expert. We use LM-Evaluation-Harness~\cite{eval-harness} to evaluate model performance across several datasets. We measure inference efficiency using vLLM~\cite{kwon2023efficient} with the Transformers backend engine on a single 40GB A100 GPU, across batch sizes and sequence lengths.

\textbf{Datasets.} We evaluate on standard benchmarks that cover general language understanding, reasoning, and knowledge recall. \textit{BoolQ} contains naturally occurring yes/no questions~\cite{clark-etal-2019-boolq}. \textit{OpenBookQA} assesses elementary science question answering~\cite{mihaylov-etal-2018-suit}, and \textit{RTE} evaluates textual entailment~\cite{dagan2005pascal}. To measure broad world knowledge, we use \textit{MMLU}, which spans 57 subjects across STEM and the humanities~\cite{hendrycks2021measuring}. We also evaluate commonsense reasoning using \textit{PIQA} for physical interactions~\cite{bisk2020piqa}, \textit{WinoGrande} for adversarial Winograd-style problems~\cite{sakaguchi2021winogrande}, and \textit{HellaSwag} for commonsense inference~\cite{zellers-etal-2019-hellaswag}. Finally, we include \textit{ARC-Challenge} and \textit{ARC-Easy} to test grade-school science question answering~\cite{clark2018think}.

\textbf{Metrics.} We report both efficiency and performance metrics. For efficiency, we measure \textit{throughput}, tokens processed or generated per second; \textit{FLOPs}, floating-point operations per inference step; and \textit{latency}, time from request to the start of generation. For performance with LM-Evaluation-Harness, we report \textit{accuracy}, the fraction of examples where the correct choice has the highest total log-likelihood; and \textit{normalized accuracy}, the fraction of examples where the correct choice has the highest length-normalized log-likelihood, computed by dividing each choice log-likelihood by its answer length in bytes

\subsection{Model Evaluation}
\label{sec:eval}

We report benchmark results in Table~\ref{tab:main_results}. Swimba-14B activates only one expert out of four, so its FLOPs are nearly the same as the Nemotron-H-8B baseline; we quantify this in the next section. Table~\ref{tab:main_results} shows that Swimba-14B outperforms Nemotron-H-8B on most tasks and achieves comparable performance on the remaining tasks. Swimba-14B also improves the average score across all tasks. Overall, these results indicate that Swimba-14B improves performance while keeping FLOPs nearly unchanged.

\subsection{Inference Efficiency}
\label{sec:efficiency}
We evaluate inference efficiency from two complementary angles. First, we analyze the theoretical compute cost of Swimba-14B and Nemotron-H-8B. Second, we benchmark end-to-end serving performance using vLLM, a production-style inference engine, to capture practical latency and throughput.

\textbf{FLOPs analysis.}
We use the PyTorch profiler~\cite{paszke2019pytorch} to estimate inference FLOPs. Specifically, we enable FLOP counting in the profiler during a forward pass; PyTorch reports per-operator FLOP estimates, which we then sum over all recorded operators to obtain the total forward-pass FLOPs. Table~\ref{tab:per_token_flops} shows that Swimba-14B and Nemotron-H-8B differ by less than $0.2\%$ in FLOPs.

\begin{figure*}[t]
\centering
\includegraphics[width=\linewidth]{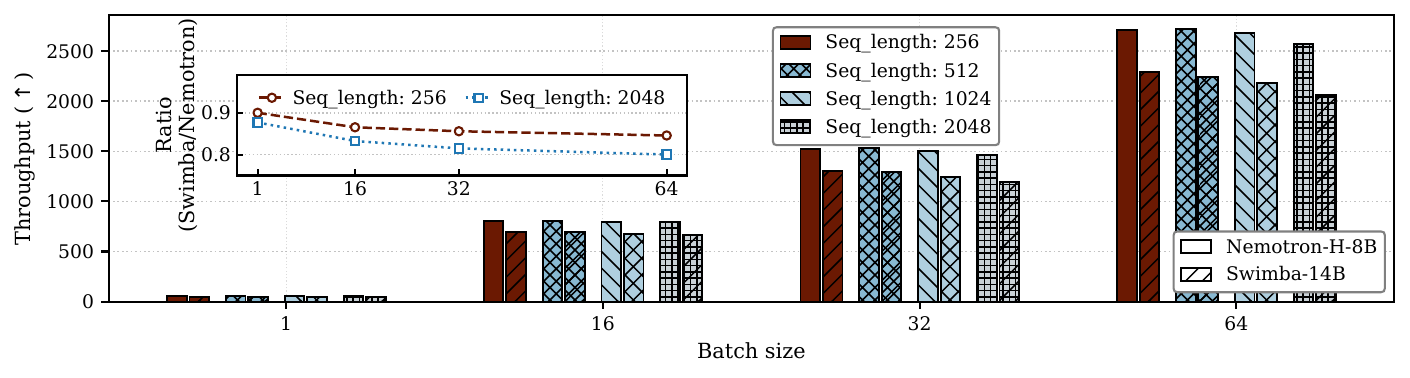}
\caption{Decoding throughput on vLLM versus batch size for multiple input-output sequence lengths, with the inset reporting the throughput ratio (Swimba/Nemotron). Swimba-14B remains close to the Nemotron-H-8B baseline in throughput with a slight drop.}
\label{fig:throughput}
\end{figure*}

\begin{figure*}[t]
\centering
\includegraphics[width=\linewidth]{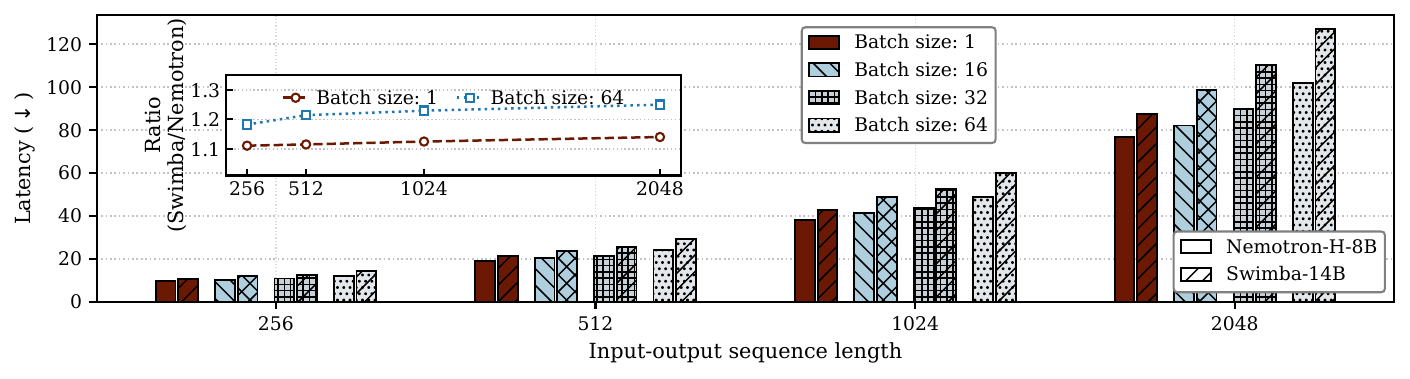}
\caption{Latency on vLLM versus input-output sequence length for multiple batch sizes, with the inset reporting the ratio (Swimba/Nemotron). Latency trends are similar across sequence lengths, demonstrating slightly worse end-to-end cost between Swimba-14B and Nemotron-H-8B.}
\label{fig:latency}
\end{figure*}

\textbf{\texttt{vLLM} throughput and latency.}
To measure practical inference efficiency, we benchmark decoding with \texttt{vLLM}~\cite{kwon2023efficient} under identical settings for the baseline and our model. Fig.~\ref{fig:throughput} reports end-to-end throughput and Fig.~\ref{fig:latency} reports latency. We sweep input-output sequence lengths and batch sizes. Across these settings, Swimba-14B has slightly lower throughput and higher latency than Nemotron-H-8B, with as low as a $10\%$ slowdown.

Despite having nearly the same FLOPs, Swimba-14B shows a larger runtime gap than FLOPs. We attribute this to routing overhead~\cite{go2025moetuner}, which can be relatively time-consuming. However, as we increase the number of experts, the parameter count grows while the throughput and latency remains largely unchanged as long as the number of activated experts is fixed~\cite{chitty2025moe}.

%% file: src/related_work.tex
\section{Related Work}
\label{sec:related_work}
\paragraph{Efficient Sequence Modeling (SSMs and RNNs).}
Structured State Space Models (SSMs) originated from the HiPPO theory~\cite{gu2020hippo} and S4~\cite{gu2021efficiently}, offering efficient alternatives to transformers~\cite{vaswani2017attention} for long sequences. This lineage evolved through {H3}~\cite{fu2022hungry} and {Hyena}~\cite{poli2023hyena}, which replaced attention with long convolutions. {Mamba}~\cite{gu2023mamba} established the efficacy of input-dependent selection, while {Mamba-2}~\cite{dao2024transformers} reformulated the architecture under the SSD framework to leverage matrix multiplication accelerators. Parallel to SSMs, modern Recurrent Neural Networks (RNNs) such as {RWKV}~\cite{peng2023rwkv}, {RetNet}~\cite{sun2023retentive}, and {Griffin/Hawk}~\cite{de2024griffin} have also achieved transformer-competitive performance by linearizing attention or gating recurrence. Our method builds strictly upon the Mamba-2 SSD specification.

\paragraph{Mixture-of-Experts (MoE) in Sequence Models.}
Sparse MoE is a dominant strategy for scaling model capacity while maintaining constant inference costs~\cite{shazeer2017outrageously}. The field has progressed from early implementations to massive scale-out with {GShard}~\cite{lepikhin2020gshard} and {Switch Transformer}~\cite{fedus2022switch} (top-1 routing). Later advances focus on routing stability and expert utilization, including {Expert Choice}~\cite{zhou2022mixture}, {Soft MoE}~\cite{puigcerver2023sparse}, and open-weights models like {Mixtral}~\cite{jiang2024mixtral}. While most approaches apply sparsity solely to the MLP (Feed-Forward) blocks, {SwitchHead}~\cite{csordas2024switchhead} and MoA (Mixture-of-Attention-Head)~\cite{zhang2022mixture} demonstrated the utility of routing within the token-mixing layers, a concept our method adapts to the SSM context.

\paragraph{MoE-SSM Hybrids.}
Combining MoE with linear-time architectures is an emerging frontier. The prevailing strategy is \textit{block-mixing}, where standard SSM layers are interleaved with MoE-MLP layers. {BlackMamba}~\cite{anthony2024blackmamba} and {MoE-Mamba}~\cite{pioro2024moe} pioneered this by replacing Mamba's internal MLP or interleaving MoE blocks. {Jamba}~\cite{lieber2024jamba} and {Zamba}~\cite{glorioso2024zamba} extended this by mixing Mamba, Attention, and MoE-MLP layers into hybrid macro-architectures. Most notably, the NVIDIA Nemotron-3 series~\cite{blakeman2025nvidia} and Nemotron-H family~\cite{blakeman2025nemotron} scale this to production levels, employing a hybrid architecture that interleaves Mamba-2 token mixers with MoE Feed-Forward Networks (FFNs). In these designs, the SSM dynamics remains dense; sparsity is extrinsic to the state transition.  {Routing Mamba (RoM)}~\cite{zhan2025routing} shifts toward \textit{inner mixing} by applying MoE to the linear projections surrounding the SSM core. However, these methods are primarily motivated by engineering considerations and do not distinguish between the two types of MoE--SSM. In contrast, our {Swimba} method performs \textit{parameter-space mixing}: we route experts to generate the specific $\mathbf{B}_t$ and $\mathbf{C}_t$ that govern the state transition itself. By sharing the transition matrix $\mathbf{A}$ and aggregating parameters rather than outputs, we maintain a single coherent state trajectory.

%% file: src/conclusion.tex
\section{Conclusion}
We incorporate mixture-of-experts (MoE) specialization into SSM, while avoiding expert-dependent recurrence cost. We first formalize two MoE–SSM families, MoE of separated SSMs and MoE-parameterized SSM, and provide theoretical characterizations of these designs. We then propose Switch Mamba (Swimba), which routes over expert-produced SSM streams but mixes them in parameter space to retain a single hidden-state trajectory. Experiments evaluate a Swimba-14B model on standard benchmarks, and show improved average performance at essentially unchanged FLOPs per token. vLLM decoding measurements indicate slightly lower throughput and higher latency, consistent with routing overhead. Overall, these results support the effectiveness of Swimba under matched-compute constraints and suggest a practical path to scaling SSM capacity via MoE.

%% file: src/appendix.tex
% \section{Appendix}
\section{Proofs for Section~\ref{sec:theory_in_method}}
\label{sec:appendix_proofs}

This appendix provides proofs for Theorems~\ref{thm:moe_structure}--\ref{thm:moe_expressive}.
Unless stated otherwise, we work with the mixed model
\begin{equation}
h_t = A h_{t-1} + \widetilde U_t,
\qquad
Y_t = \widetilde C_t^\top h_t,
\label{eq:app_mixed_ssm}
\end{equation}
where
\begin{equation}
\widetilde U_t := \sum_{e\in\mathcal{K}_t}\pi_{t,e}\,B_t^{(e)}X_t^{(e)},
\qquad
\widetilde C_t := \sum_{e\in\mathcal{K}_t}\pi_{t,e}\,C_t^{(e)}.
\label{eq:app_mixed_streams}
\end{equation}
For the separated model we use
\begin{equation}
h_t^{(e)} = A h_{t-1}^{(e)} + B_t^{(e)}X_t^{(e)},
\qquad
Y_t^{\mathrm{sep}} = \sum_{e\in\mathcal{K}_t}\pi_{t,e}\,(C_t^{(e)})^\top h_t^{(e)}.
\label{eq:app_sep_ssm}
\end{equation}
All statements are channelwise-compatible; one may read $h_t\in\mathbb{R}^{N\times P}$ and apply the same arguments independently to each channel, or equivalently treat a fixed channel $p$ and drop the subscript $p$.

\subsection{Proof of Theorem~\ref{thm:moe_structure}}
\begin{proof}
Fix $T,N,P,E$ and suppose $A$, $\{B_t^{(e)},C_t^{(e)},X_t^{(e)}\}$, $\{\pi_t\}$, and $\{\mathcal{K}_t\}$ are given as in the theorem.
Define $\widetilde U_t$ and $\widetilde C_t$ by \eqref{eq:app_mixed_streams}.
Then the proposed mixed update \eqref{eq:moe_state_method} can be rewritten as
\[
h_t = A h_{t-1} + \widetilde U_t,
\]
and the mixed readout \eqref{eq:moe_out_method} can be rewritten as
\[
Y_t = \widetilde C_t^\top h_t.
\]
This is exactly the defining form of a (time-varying) selective SSM with state size $N$ driven by the input stream $\widetilde U_t$ and read out by $\widetilde C_t$.
The state dimension is the dimension of $h_t$, which is $N$ by construction, and this does not depend on $E$ or on the cardinality of $\mathcal{K}_t$.
\end{proof}

\subsection{Proof of Theorem~\ref{thm:moe_complexity}}
\begin{proof}
The forward computation decomposes into two parts.

First, the recurrence in \eqref{eq:moe_state_method} (or equivalently \eqref{eq:app_mixed_ssm}) advances a single state trajectory for $T$ steps.
By definition, each step costs $\mathcal{C}_{\mathrm{step}}(N,P)$, hence the total cost of state evolution is $T\,\mathcal{C}_{\mathrm{step}}(N,P)$.

Second, at each time $t$, the mixed injection term and mixed readout require forming the sums in \eqref{eq:moe_state_method}--\eqref{eq:moe_out_method}.
When at most $k$ experts are active, these sums cost $\mathcal{C}_{\mathrm{mix}}(k,P)$ by assumption, and therefore the total mixing cost is $T\,\mathcal{C}_{\mathrm{mix}}(k,P)$.

Combining the two contributions yields
\[
\mathcal{O}\!\left(T\,\mathcal{C}_{\mathrm{step}}(N,P)+T\,\mathcal{C}_{\mathrm{mix}}(k,P)\right).
\]

For the separated design, one advances $E$ state trajectories per step, each with cost $\mathcal{C}_{\mathrm{step}}(N,P)$, giving a total recurrence cost $T\,E\,\mathcal{C}_{\mathrm{step}}(N,P)$.
The remaining costs (routing and output mixing) are lower order relative to the repeated recurrence and do not change the stated scaling.
\end{proof}

\subsection{Proof of Theorem~\ref{thm:moe_stability}}
\begin{proof}
Assume $\|A\|\le\rho<1$ and $\|\widetilde U_t\|\le U$ for all $t$.
Unrolling the recursion \eqref{eq:app_mixed_ssm} gives
\[
h_t = A^t h_0 + \sum_{i=1}^{t} A^{t-i}\widetilde U_i.
\]
Taking norms and using submultiplicativity yields
\[
\|h_t\|
\le \|A^t\|\,\|h_0\| + \sum_{i=1}^{t} \|A^{t-i}\|\,\|\widetilde U_i\|
\le \rho^t\|h_0\| + U\sum_{i=1}^{t}\rho^{t-i}.
\]
The geometric series satisfies $\sum_{i=1}^{t}\rho^{t-i}=\sum_{j=0}^{t-1}\rho^{j}=(1-\rho^{t})/(1-\rho)$, hence
\[
\|h_t\|\le \rho^t\|h_0\| + \frac{1-\rho^t}{1-\rho}\,U.
\]
For the output, we have $Y_t=\widetilde C_t^\top h_t$ and thus
\[
\|Y_t\|\le \|\widetilde C_t\|\,\|h_t\|\le C\|h_t\|.
\]
Substituting the bound on $\|h_t\|$ completes the proof.
\end{proof}

\subsection{Proof of Theorem~\ref{thm:moe_approx}}
\begin{proof}
We first prove the equality regime.
Assume $\pi_{t,e}\equiv \pi_e$ is time-invariant and $C_t^{(e)}\equiv C_t$ for all $e$.
Define the weighted average state
\[
\bar h_t := \sum_{e=1}^{E}\pi_e\,h_t^{(e)}.
\]
Using the separated recurrence \eqref{eq:app_sep_ssm} and linearity,
\[
\bar h_t
= \sum_{e=1}^{E}\pi_e\left(A h_{t-1}^{(e)} + B_t^{(e)}X_t^{(e)}\right)
= A\sum_{e=1}^{E}\pi_e h_{t-1}^{(e)} + \sum_{e=1}^{E}\pi_e B_t^{(e)}X_t^{(e)}
= A\bar h_{t-1} + \widetilde U_t,
\]
where $\widetilde U_t=\sum_e \pi_e B_t^{(e)}X_t^{(e)}$ matches \eqref{eq:app_mixed_streams} with $\mathcal{K}_t=[E]$ (or with any $\mathcal{K}_t$ that contains exactly those $e$ for which $\pi_e\neq 0$).
Since $h_0^{(e)}=0$ for all $e$, we have $\bar h_0=0$.
The mixed model starts from $h_0=0$ as well, so by uniqueness of solutions to the linear recursion, $h_t=\bar h_t$ for all $t$.
For the outputs,
\[
Y_t^{\mathrm{sep}}
= \sum_{e=1}^{E}\pi_e\,C_t^\top h_t^{(e)}
= C_t^\top\sum_{e=1}^{E}\pi_e h_t^{(e)}
= C_t^\top \bar h_t
= C_t^\top h_t
= Y_t^{\mathrm{mix}},
\]
which proves equality.

We next prove the stated mismatch bound.
Assume $\|C_t^{(e)}\|\le C$ for all $t,e$.
Write the difference as
\begin{align*}
Y_t^{\mathrm{sep}}-Y_t^{\mathrm{mix}}
&= \sum_{e\in\mathcal{K}_t}\pi_{t,e}\,(C_t^{(e)})^\top h_t^{(e)}
- \left(\sum_{e\in\mathcal{K}_t}\pi_{t,e}C_t^{(e)}\right)^\top h_t \\
&= \sum_{e\in\mathcal{K}_t}\pi_{t,e}\,(C_t^{(e)})^\top\left(h_t^{(e)}-h_t\right).
\end{align*}
Taking norms and applying the triangle inequality yields
\[
\|Y_t^{\mathrm{sep}}-Y_t^{\mathrm{mix}}\|
\le \sum_{e\in\mathcal{K}_t}\pi_{t,e}\,\|C_t^{(e)}\|\,\|h_t^{(e)}-h_t\|
\le C\sum_{e\in\mathcal{K}_t}\pi_{t,e}\,\|h_t^{(e)}-h_t\|.
\]
Extending the sum over all experts (with $\pi_{t,e}=0$ outside $\mathcal{K}_t$) gives the bound as stated.

Finally, define $\delta_t^{(e)}:=h_t^{(e)}-h_t$.
Subtracting \eqref{eq:app_mixed_ssm} from \eqref{eq:app_sep_ssm} gives
\[
\delta_t^{(e)}
= A\delta_{t-1}^{(e)} + \left(B_t^{(e)}X_t^{(e)}-\widetilde U_t\right),
\]
which is a stable linear recursion when $\|A\|\le\rho<1$.
This identifies the driving term as the deviation of each expert injection from the mixture injection.
\end{proof}

\subsection{Proof of Theorem~\ref{thm:moe_expressive}}
\begin{proof}
We give a concrete one-dimensional construction and then show that it cannot be represented by any single-expert linear-projection model under the stated specialization.

Consider the setting $T=1$, $N=P=1$, and $h_0=0$.
Let $A=0$.
Let there be $E=2$ experts and take $\mathcal{K}_1=\{1,2\}$.
Define the expert streams so that the injected term is constant within each expert:
\[
B_1^{(1)}X_1^{(1)} \equiv 0,
\qquad
B_1^{(2)}X_1^{(2)} \equiv 1,
\qquad
C_1^{(1)}\equiv C_1^{(2)}\equiv 1.
\]
This choice is compatible with the assumption that these quantities are produced by linear projections of $X_1$ by simply taking projections that ignore $X_1$ and output constants (equivalently, linear maps applied to an augmented feature vector with a fixed bias coordinate; if biases are disallowed, the same construction can be obtained by restricting to inputs with an appended constant feature).
With this choice, the mixed model \eqref{eq:moe_state_method}--\eqref{eq:moe_out_method} becomes
\[
h_1 = \pi_{1,2},
\qquad
Y_1 = h_1 = \pi_{1,2}.
\]
Now choose the router logits as $g_{1,1}=0$ and $g_{1,2}=x$, where $x\in\mathbb{R}$ denotes the scalar input feature (that is, $X_1=x$) and $g_{1,2}$ is a linear function of $X_1$.
Then
\[
\pi_{1,2} = \frac{e^{x}}{1+e^{x}} =: \sigma(x),
\]
so the MoE-parameterized layer represents the logistic sigmoid function $Y_1=\sigma(x)$.

We now argue that no single-expert linear-projection layer in this one-step specialization can represent $\sigma(x)$.
Under the same specialization ($T=1$, $h_0=0$, $A=0$, $N=P=1$) and with one expert, the layer output takes the form
\[
Y_1 = C_1(X_1)\,B_1(X_1)\,X_1'(X_1),
\]
where $B_1(\cdot)$, $C_1(\cdot)$, and $X_1'(\cdot)$ are linear functions of $X_1$ by assumption.
Therefore $Y_1$ is a polynomial in $x$ of degree at most $3$.
In particular, $Y_1$ is an entire function whose third derivative is constant and whose fourth derivative is identically zero.

In contrast, $\sigma(x)$ is not a polynomial: its derivative satisfies $\sigma'(x)=\sigma(x)(1-\sigma(x))$, which is nonzero for all $x\in\mathbb{R}$ and decays to $0$ only as $x\to\pm\infty$.
A nonconstant polynomial cannot have this behavior; more directly, any nonzero polynomial diverges in magnitude as $|x|\to\infty$, while $\sigma(x)$ remains bounded in $(0,1)$.
Hence $\sigma(x)$ cannot equal any polynomial on $\mathbb{R}$, and therefore cannot be represented by the single-expert linear-projection model in this specialization.

This exhibits a function realizable by the two-expert MoE-parameterized layer but not by any single-expert linear-projection layer, while both execute one state update.
\end{proof}

%% file: reference.bib
@string{AAAI = "National Conference on Artificial Intelligence (AAAI)"}

@inproceedings{gu2023mamba,
  title={Mamba: Linear-time sequence modeling with selective state spaces},
  author={Gu, Albert and Dao, Tri},
  booktitle={First conference on language modeling},
  year={2024}
}

@article{dao2024transformers,
  title={Transformers are ssms: Generalized models and efficient algorithms through structured state space duality},
  author={Dao, Tri and Gu, Albert},
  journal={arXiv preprint arXiv:2405.21060},
  year={2024}
}

@article{gu2021efficiently,
  title={Efficiently modeling long sequences with structured state spaces},
  author={Gu, Albert and Goel, Karan and R{\'e}, Christopher},
  journal={arXiv preprint arXiv:2111.00396},
  year={2021}
}

@article{shazeer2017outrageously,
  title={Outrageously large neural networks: The sparsely-gated mixture-of-experts layer},
  author={Shazeer, Noam and Mirhoseini, Azalia and Maziarz, Krzysztof and Davis, Andy and Le, Quoc and Hinton, Geoffrey and Dean, Jeff},
  journal={arXiv preprint arXiv:1701.06538},
  year={2017}
}

@article{fedus2022switch,
  title={Switch transformers: Scaling to trillion parameter models with simple and efficient sparsity},
  author={Fedus, William and Zoph, Barret and Shazeer, Noam},
  journal={Journal of Machine Learning Research},
  volume={23},
  number={120},
  pages={1--39},
  year={2022}
}

@article{anthony2024blackmamba,
  title={Blackmamba: Mixture of experts for state-space models},
  author={Anthony, Quentin and Tokpanov, Yury and Glorioso, Paolo and Millidge, Beren},
  journal={arXiv preprint arXiv:2402.01771},
  year={2024}
}

@article{pioro2024moe,
  title={Moe-mamba: Efficient selective state space models with mixture of experts},
  author={Pi{\'o}ro, Maciej and Ciebiera, Kamil and Kr{\'o}l, Krystian and Ludziejewski, Jan and Krutul, Micha{\l} and Krajewski, Jakub and Antoniak, Szymon and Mi{\l}o{\'s}, Piotr and Cygan, Marek and Jaszczur, Sebastian},
  journal={arXiv preprint arXiv:2401.04081},
  year={2024}
}

@article{lieber2024jamba,
  title={Jamba: A hybrid transformer-mamba language model},
  author={Lieber, Opher and Lenz, Barak and Bata, Hofit and Cohen, Gal and Osin, Jhonathan and Dalmedigos, Itay and Safahi, Erez and Meirom, Shaked and Belinkov, Yonatan and Shalev-Shwartz, Shai and others},
  journal={arXiv preprint arXiv:2403.19887},
  year={2024}
}

@article{zhan2025routing,
  title={Routing Mamba: Scaling State Space Models with Mixture-of-Experts Projection},
  author={Zhan, Zheng and Ren, Liliang and Wang, Shuohang and Liu, Liyuan and Liu, Yang and Gong, Yeyun and Wang, Yanzhi and Shen, Yelong},
  journal={arXiv preprint arXiv:2506.18145},
  year={2025}
}

@article{fu2022hungry,
  title={Hungry hungry hippos: Towards language modeling with state space models},
  author={Fu, Daniel Y and Dao, Tri and Saab, Khaled K and Thomas, Armin W and Rudra, Atri and R{\'e}, Christopher},
  journal={arXiv preprint arXiv:2212.14052},
  year={2022}
}

@article{lepikhin2020gshard,
  title={Gshard: Scaling giant models with conditional computation and automatic sharding},
  author={Lepikhin, Dmitry and Lee, HyoukJoong and Xu, Yuanzhong and Chen, Dehao and Firat, Orhan and Huang, Yanping and Krikun, Maxim and Shazeer, Noam and Chen, Zhifeng},
  journal={arXiv preprint arXiv:2006.16668},
  year={2020}
}

@article{ghahramani2000variational,
  title={Variational learning for switching state-space models},
  author={Ghahramani, Zoubin and Hinton, Geoffrey E},
  journal={Neural computation},
  volume={12},
  number={4},
  pages={831--864},
  year={2000},
  publisher={MIT Press}
}

@article{gu2020hippo,
  title={Hippo: Recurrent memory with optimal polynomial projections},
  author={Gu, Albert and Dao, Tri and Ermon, Stefano and Rudra, Atri and R{\'e}, Christopher},
  journal={Advances in neural information processing systems},
  volume={33},
  pages={1474--1487},
  year={2020}
}

@inproceedings{poli2023hyena,
  title={Hyena hierarchy: Towards larger convolutional language models},
  author={Poli, Michael and Massaroli, Stefano and Nguyen, Eric and Fu, Daniel Y and Dao, Tri and Baccus, Stephen and Bengio, Yoshua and Ermon, Stefano and R{\'e}, Christopher},
  booktitle={International Conference on Machine Learning},
  pages={28043--28078},
  year={2023},
  organization={PMLR}
}

@article{peng2023rwkv,
  title={Rwkv: Reinventing rnns for the transformer era},
  author={Peng, Bo and Alcaide, Eric and Anthony, Quentin and Albalak, Alon and Arcadinho, Samuel and Biderman, Stella and Cao, Huanqi and Cheng, Xin and Chung, Michael and Grella, Matteo and others},
  journal={arXiv preprint arXiv:2305.13048},
  year={2023}
}

@article{sun2023retentive,
  title={Retentive network: A successor to transformer for large language models},
  author={Sun, Yutao and Dong, Li and Huang, Shaohan and Ma, Shuming and Xia, Yuqing and Xue, Jilong and Wang, Jianyong and Wei, Furu},
  journal={arXiv preprint arXiv:2307.08621},
  year={2023}
}

@article{de2024griffin,
  title={Griffin: Mixing gated linear recurrences with local attention for efficient language models},
  author={De, Soham and Smith, Samuel L and Fernando, Anushan and Botev, Aleksandar and Cristian-Muraru, George and Gu, Albert and Haroun, Ruba and Berrada, Leonard and Chen, Yutian and Srinivasan, Srivatsan and others},
  journal={arXiv preprint arXiv:2402.19427},
  year={2024}
}

@article{jiang2024mixtral,
  title={Mixtral of experts},
  author={Jiang, Albert Q and Sablayrolles, Alexandre and Roux, Antoine and Mensch, Arthur and Savary, Blanche and Bamford, Chris and Chaplot, Devendra Singh and Casas, Diego de las and Hanna, Emma Bou and Bressand, Florian and others},
  journal={arXiv preprint arXiv:2401.04088},
  year={2024}
}

@article{zhou2022mixture,
  title={Mixture-of-experts with expert choice routing},
  author={Zhou, Yanqi and Lei, Tao and Liu, Hanxiao and Du, Nan and Huang, Yanping and Zhao, Vincent and Dai, Andrew M and Le, Quoc V and Laudon, James and others},
  journal={Advances in Neural Information Processing Systems},
  volume={35},
  pages={7103--7114},
  year={2022}
}

@article{puigcerver2023sparse,
  title={From sparse to soft mixtures of experts},
  author={Puigcerver, Joan and Riquelme, Carlos and Mustafa, Basil and Houlsby, Neil},
  journal={arXiv preprint arXiv:2308.00951},
  year={2023}
}

@article{csordas2024switchhead,
  title={Switchhead: Accelerating transformers with mixture-of-experts attention},
  author={Csord{\'a}s, R{\'o}bert and Pikekos, Piotr and Irie, Kazuki and Schmidhuber, J{\"u}rgen},
  journal={Advances in Neural Information Processing Systems},
  volume={37},
  pages={74411--74438},
  year={2024}
}

@article{glorioso2024zamba,
  title={Zamba: A compact 7b ssm hybrid model},
  author={Glorioso, Paolo and Anthony, Quentin and Tokpanov, Yury and Whittington, James and Pilault, Jonathan and Ibrahim, Adam and Millidge, Beren},
  journal={arXiv preprint arXiv:2405.16712},
  year={2024}
}

@article{blakeman2025nvidia,
  title={NVIDIA Nemotron 3: Efficient and Open Intelligence},
  author={Blakeman, Aaron and Grattafiori, Aaron and Basant, Aarti and Gupta, Abhibha and Khattar, Abhinav and Renduchintala, Adi and Vavre, Aditya and Shukla, Akanksha and Bercovich, Akhiad and Ficek, Aleksander and others},
  journal={arXiv preprint arXiv:2512.20856},
  year={2025}
}

@article{blakeman2025nemotron,
  title={Nemotron-h: A family of accurate and efficient hybrid mamba-transformer models},
  author={Blakeman, Aaron and Basant, Aarti and Khattar, Abhinav and Renduchintala, Adithya and Bercovich, Akhiad and Ficek, Aleksander and Bjorlin, Alexis and Taghibakhshi, Ali and Deshmukh, Amala Sanjay and Mahabaleshwarkar, Ameya Sunil and others},
  journal={arXiv preprint arXiv:2504.03624},
  year={2025}
}

@article{bercovich2025llama,
  title={Llama-nemotron: Efficient reasoning models},
  author={Bercovich, Akhiad and Levy, Itay and Golan, Izik and Dabbah, Mohammad and El-Yaniv, Ran and Puny, Omri and Galil, Ido and Moshe, Zach and Ronen, Tomer and Nabwani, Najeeb and others},
  journal={arXiv preprint arXiv:2505.00949},
  year={2025}
}

@inproceedings{clark-etal-2019-boolq,
    title = "{B}ool{Q}: Exploring the Surprising Difficulty of Natural Yes/No Questions",
    author = "Clark, Christopher  and
      Lee, Kenton  and
      Chang, Ming-Wei  and
      Kwiatkowski, Tom  and
      Collins, Michael  and
      Toutanova, Kristina",
    editor = "Burstein, Jill  and
      Doran, Christy  and
      Solorio, Thamar",
    booktitle = "Proceedings of the 2019 Conference of the North {A}merican Chapter of the Association for Computational Linguistics: Human Language Technologies, Volume 1 (Long and Short Papers)",
    month = jun,
    year = "2019",
    address = "Minneapolis, Minnesota",
    publisher = "Association for Computational Linguistics",
    url = "https://aclanthology.org/N19-1300/",
    doi = "10.18653/v1/N19-1300",
    pages = "2924--2936"
}

@inproceedings{mihaylov-etal-2018-suit,
    title = "Can a Suit of Armor Conduct Electricity? A New Dataset for Open Book Question Answering",
    author = "Mihaylov, Todor  and
      Clark, Peter  and
      Khot, Tushar  and
      Sabharwal, Ashish",
    editor = "Riloff, Ellen  and
      Chiang, David  and
      Hockenmaier, Julia  and
      Tsujii, Jun{'}ichi",
    booktitle = "Proceedings of the 2018 Conference on Empirical Methods in Natural Language Processing",
    month = oct # "-" # nov,
    year = "2018",
    address = "Brussels, Belgium",
    publisher = "Association for Computational Linguistics",
    url = "https://aclanthology.org/D18-1260/",
    doi = "10.18653/v1/D18-1260",
    pages = "2381--2391"
}

@inproceedings{dagan2005pascal,
  title={The pascal recognising textual entailment challenge},
  author={Dagan, Ido and Glickman, Oren and Magnini, Bernardo},
  booktitle={Machine learning challenges workshop},
  pages={177--190},
  year={2005},
  organization={Springer}
}

@inproceedings{
hendrycks2021measuring,
title={Measuring Massive Multitask Language Understanding},
author={Dan Hendrycks and Collin Burns and Steven Basart and Andy Zou and Mantas Mazeika and Dawn Song and Jacob Steinhardt},
booktitle={International Conference on Learning Representations},
year={2021},
url={https://openreview.net/forum?id=d7KBjmI3GmQ}
}

@inproceedings{bisk2020piqa,
  title={Piqa: Reasoning about physical commonsense in natural language},
  author={Bisk, Yonatan and Zellers, Rowan and Gao, Jianfeng and Choi, Yejin and others},
  booktitle={Proceedings of the AAAI conference on artificial intelligence},
  volume={34},
  pages={7432--7439},
  year={2020}
}

@article{sakaguchi2021winogrande,
  title={Winogrande: An adversarial winograd schema challenge at scale},
  author={Sakaguchi, Keisuke and Bras, Ronan Le and Bhagavatula, Chandra and Choi, Yejin},
  journal={Communications of the ACM},
  volume={64},
  number={9},
  pages={99--106},
  year={2021},
  publisher={ACM New York, NY, USA}
}

@inproceedings{zellers-etal-2019-hellaswag,
    title = "{H}ella{S}wag: Can a Machine Really Finish Your Sentence?",
    author = "Zellers, Rowan  and
      Holtzman, Ari  and
      Bisk, Yonatan  and
      Farhadi, Ali  and
      Choi, Yejin",
    editor = "Korhonen, Anna  and
      Traum, David  and
      M{\`a}rquez, Llu{\'i}s",
    booktitle = "Proceedings of the 57th Annual Meeting of the Association for Computational Linguistics",
    month = jul,
    year = "2019",
    address = "Florence, Italy",
    publisher = "Association for Computational Linguistics",
    url = "https://aclanthology.org/P19-1472/",
    doi = "10.18653/v1/P19-1472",
    pages = "4791--4800"
}

@article{clark2018think,
  title={Think you have solved question answering? try arc, the ai2 reasoning challenge},
  author={Clark, Peter and Cowhey, Isaac and Etzioni, Oren and Khot, Tushar and Sabharwal, Ashish and Schoenick, Carissa and Tafjord, Oyvind},
  journal={arXiv preprint arXiv:1803.05457},
  year={2018}
}

@misc{eval-harness,
  author       = {Gao, Leo and Tow, Jonathan and Abbasi, Baber and Biderman, Stella and Black, Sid and DiPofi, Anthony and Foster, Charles and Golding, Laurence and Hsu, Jeffrey and Le Noac'h, Alain and Li, Haonan and McDonell, Kyle and Muennighoff, Niklas and Ociepa, Chris and Phang, Jason and Reynolds, Laria and Schoelkopf, Hailey and Skowron, Aviya and Sutawika, Lintang and Tang, Eric and Thite, Anish and Wang, Ben and Wang, Kevin and Zou, Andy},
  title        = {The Language Model Evaluation Harness},
  month        = 07,
  year         = 2024,
  publisher    = {Zenodo},
  version      = {v0.4.3},
  doi          = {10.5281/zenodo.12608602},
  url          = {https://zenodo.org/records/12608602}
}

@inproceedings{kwon2023efficient,
  title={Efficient Memory Management for Large Language Model Serving with PagedAttention},
  author={Woosuk Kwon and Zhuohan Li and Siyuan Zhuang and Ying Sheng and Lianmin Zheng and Cody Hao Yu and Joseph E. Gonzalez and Hao Zhang and Ion Stoica},
  booktitle={Proceedings of the ACM SIGOPS 29th Symposium on Operating Systems Principles},
  year={2023}
}

@article{paszke2019pytorch,
  title={Pytorch: An imperative style, high-performance deep learning library},
  author={Paszke, Adam and Gross, Sam and Massa, Francisco and Lerer, Adam and Bradbury, James and Chanan, Gregory and Killeen, Trevor and Lin, Zeming and Gimelshein, Natalia and Antiga, Luca and others},
  journal={Advances in neural information processing systems},
  volume={32},
  year={2019}
}

@article{zhang2022mixture,
  title={Mixture of attention heads: Selecting attention heads per token},
  author={Zhang, Xiaofeng and Shen, Yikang and Huang, Zeyu and Zhou, Jie and Rong, Wenge and Xiong, Zhang},
  journal={arXiv preprint arXiv:2210.05144},
  year={2022}
}

@inproceedings{chitty2025moe,
  title={MoE-Inference-Bench: Performance Evaluation of Mixture of Expert Large Language and Vision Models},
  author={Chitty-Venkata, Krishna Teja and Howland, Sylvia and Azar, Golara and Soboleva, Daria and Vassilieva, Natalia and Raskar, Siddhisanket and Emani, Murali and Vishwanath, Venkatram},
  booktitle={Proceedings of the SC'25 Workshops of the International Conference for High Performance Computing, Networking, Storage and Analysis},
  pages={1502--1511},
  year={2025}
}

@article{go2025moetuner,
  title={Moetuner: Optimized mixture of expert serving with balanced expert placement and token routing},
  author={Go, Seokjin and Mahajan, Divya},
  journal={arXiv preprint arXiv:2502.06643},
  year={2025}
}

@article{vaswani2017attention,
  title={Attention is all you need},
  author={Vaswani, Ashish and Shazeer, Noam and Parmar, Niki and Uszkoreit, Jakob and Jones, Llion and Gomez, Aidan N and Kaiser, {\L}ukasz and Polosukhin, Illia},
  journal={Advances in neural information processing systems},
  volume={30},
  year={2017}
}
